\documentclass[twocolumn,10pt]{asme2ej}
       
\usepackage{overpic,rotating,contour,footnote}                        
\usepackage{graphics,subfigure,enumerate}
\usepackage{graphicx,psfrag}
\usepackage{epsfig}  
\usepackage{amssymb,amsmath} 
\usepackage[all]{xy}
\usepackage[color=aliceblue]{todonotes}

\usepackage{xcolor}
\definecolor{aliceblue}{rgb}{0.94, 0.97, 1.0}

\usepackage{mathptmx}       
\usepackage{helvet,overpic}         
\usepackage{courier}        
\usepackage{graphicx,subfigure}        
\usepackage[bottom]{footmisc} 
\usepackage{amsmath,amssymb,epsfig} 
      
\def\Beweisende{\square}            
\def\BewEnde{\hfill{\Beweisende}}

\def\phm{{\hphantom{-}}} 

\def\phi{\varphi}

\def\CC{{\mathbb C}}
\def\RR{{\mathbb R}}

\def\PP{{\mathbb P}}

\def\Vkt#1{{\mathbf #1}} 

\newcommand{\go}[1]{{\sf #1}}

\newtheorem{thm}{Theorem}
\newtheorem{lem}{Lemma}
\newtheorem{rmk}{Remark} 

\newtheorem{assumption}{Assumption}

\newtheorem{cor}{Corollary}

\title{Addendum to Pentapods with Mobility 2}
\author{Georg Nawratil
    \affiliation{
	Institute of Discrete Mathematics and Geometry\\
	Vienna University of Technology\\
	Wiedner Hauptstrasse 8-10/104, Vienna 1040, Austria\\
  Email: nawratil@geometrie.tuwien.ac.at
    }	  
}
\author{Josef Schicho
    \affiliation{
	Johann Radon Institute for Computational and Applied Mathematics\\
	Austrian Academy of Sciences\\
	Altenberger Strasse 69, Linz 4040, Austria\\
    Email: josef.schicho@ricam.oeaw.ac.at
    }	 
}

\begin{document}

\maketitle    

\begin{abstract}
{\it
In a foregoing publication the authors studied pentapods with mobility 2, where neither all platform anchor points nor 
all base anchor points are located on a line. It turned out that the given classification is incomplete. 
This addendum is devoted to the discussion of the missing cases resulting in additional solutions already known to Duporcq.}
\end{abstract}

\section{Introduction}\label{intro}

The geometry of a pentapod is given by the five base anchor points $\go M_i$ with coordinates
$\Vkt M_i:=(A_i,B_i,C_i)^T$ with respect to the fixed system and by the five platform anchor points $\go m_i$ 
with coordinates $\Vkt m_i:=(a_i,b_i,c_i)^T$ with respect to the moving system (for $i=1,\ldots ,5$). 
Each pair $(\go M_i,\go m_i)$ of corresponding anchor points is connected by a SPS-leg, where only 
the prismatic joint (P) is active and the spherical joints (S) are passive. 

If the geometry of the manipulator is given, as well as the lengths of the five pairwise distinct legs, the 
pentapod has generically mobility 1 according to the formula of Gr\"ubler. The corresponding motion is called a 
1-dimensional self-motion of the pentapod. But, under particular conditions, the manipulator can 
gain additional mobility.  We can focus on pentapods with mobility 2, as those with higher-dimensional 
self-motions are already known (cf.\ \cite[Corollary 1]{asme_ns}).

\subsection{Reason for the Addendum}\label{reason}

The classification of pentapods with mobility 2 given in \cite{asme_ns}
was based on the following theorem of \cite{gns2}: 

\begin{thm}\label{thm1a}
If the mobility of a pentapod is 2 or higher, then one of the following conditions holds
\footnote{After a possible necessary renumbering of anchor points and exchange of the platform and the base.}:
        \begin{enumerate}[(a)]
        \item  
        The platform and the base are similar. This is a so-called equiform pentapod.
        \item
        The platform and the base are planar and affine equivalent. This is a so-called planar affine pentapod.
        \item
        There exists $p\leq 5$ such that $\go m_1,\ldots,\go m_p$ are collinear and 
        $\go M_{p+1},\ldots ,\go M_5$ are equal; i.e.,  $\go M_{p+1}=\ldots =\go M_5$.
        \item
        $\go M_1,\go M_2,\go M_3$ are located on the line $\go g$ which is parallel to the line $\go h$ spanned 
				by $\go M_4$ and $\go M_5$. Moreover $\go m_1,\go m_2,\go m_3$ are located on the line $\go g^{\prime}$ 
				which is parallel to the line $\go h^{\prime}$ spanned 
				by $\go m_4$ and $\go m_5$.
        \end{enumerate}
\end{thm} 
During the literature research for the article \cite[Section 1]{icosapod}, we came across 
the work \cite{duporcq} of Duporcq, which describes the 
following remarkable motion (see Fig.\ \ref{fig:complete_quadrilaterals}):

{\it Let $\go M_1,\ldots ,\go M_6$ and $\go m_1, \ldots, \go m_6$ be the 
vertices of two complete quadrilaterals, which are congruent. Moreover the 
vertices are labeled in a way that $\go m_i$ is the opposite vertex of 
$\go M_i$ for $i \in \{1, \dotsc, 6\}$.  
Then there exist a two-parametric line-symmetric motion where 
each $\go m_i$ is running on spheres centered in $\go M_i$.}

\begin{figure}[h!] 
\begin{center}   
  \begin{overpic}[width=40mm]{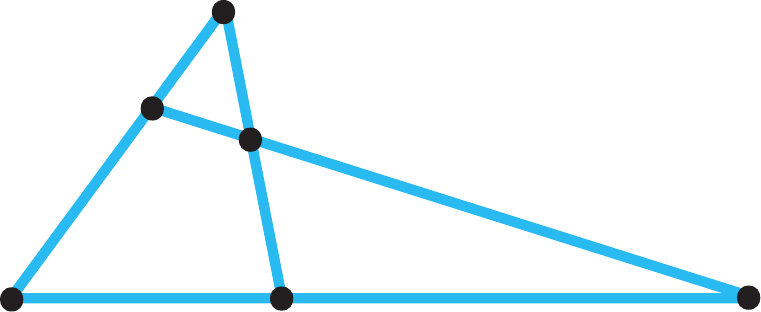}
    \begin{small}
      \put(9,5){$\go M_1$}
      \put(39,5){$\go M_2$}
      \put(92.5,6){$\go M_3$}
      \put(7,27){$\go M_5$}
      \put(35,25){$\go M_4$}
      \put(33,38){$\go M_6$}
    \end{small}
  \end{overpic} 
  \,
  \begin{overpic}[width=40mm]{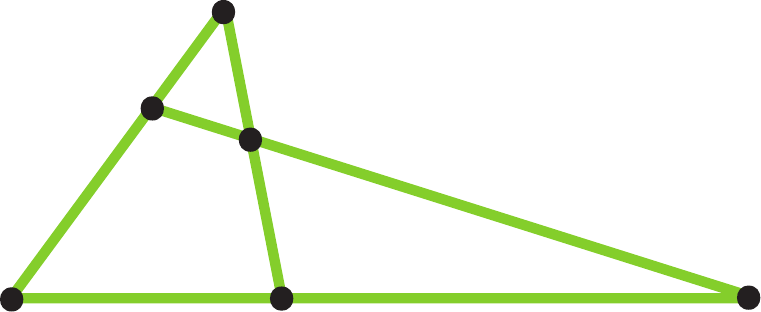}
    \begin{small}
      \put(8.5,5){$\go m_4$}
      \put(39,5){$\go m_5$}
      \put(93,6){$\go m_6$}
      \put(7,27){$\go m_2$}
      \put(35,25){$\go m_1$}
      \put(33,38){$\go m_3$}
    \end{small}
  \end{overpic} 
\end{center}
\caption{Illustration of Duporcq's complete quadrilaterals.
}
  \label{fig:complete_quadrilaterals}
\end{figure}
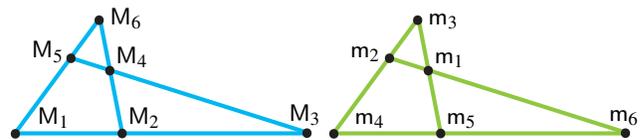     

It can easily be checked that this configuration of anchor points 
corresponds to an architecturally singular hexapod (e.g.~\cite{roschel} 
or~\cite{karger_arch}). As architecturally singular manipulators are 
redundant we can remove any leg --- without loss of generality (w.l.o.g.) we suppose that this is the sixth 
leg --- without changing the direct kinematics of the mechanism. Therefore 
the resulting pentapod $\go M_1,\ldots ,\go M_5$ and 
$\go m_1,\ldots ,\go m_5$, which we call a {\it Duporcq pentapod} 
for short, has also a two-parametric line-symmetric self-motion. 
This yields a counter-example to Theorem \ref{thm1a}, 
but the flaw can be fixed by adding the following case to Theorem \ref{thm1a} (cf.\ \cite{erratum}):
{\it
        \begin{enumerate}[(e)]
        \item  
				The following triples of points are collinear:
				\begin{equation*}
				\go M_1,\go M_2,\go M_3, \quad
				\go M_3,\go M_4,\go M_5, \quad
				\go m_3,\go m_1,\go m_i, \quad
				\go m_3,\go m_j,\go m_k,
				\end{equation*}
				with pairwise distinct $i,j,k\in\left\{2,4,5\right\}$. Moreover the points $\go M_1,\ldots ,\go M_5$ 
				are pairwise distinct as well as the points $\go m_1,\ldots ,\go m_5$. 
				\end{enumerate}
}
As Theorem \ref{thm1a} only gives necessary conditions, 
the addendum is devoted to the determination of sufficient ones  
for the 2-dimensional mobility of pentapods belonging to item (e).  
In detail the paper is structured as follows:

In Section \ref{sec:moeb} further necessary conditions are obtained by 
means of M\"obius photogrammetry, which restrict the pentapod designs of (e) to 
three possible cases, up to affinities of the planar platform and the planar base. 
In Section \ref{bonds} we repeat the theory of bonds based on two different embeddings of
$SE(3)$ and prove Lemmata \ref{lem:0} and \ref{lem:1} as well as Corollaries \ref{cor:0} and \ref{cor:1}. 
Based on these results we show in Section \ref{comp} that only the 
Duporcq pentapods were missed by our classification given in \cite{asme_ns}. 
The consequence of this result for article \cite{asme_ns} are summed up in the conclusions 
(Section \ref{sec:conclusion}).


\section{M\"obius photogrammetric considerations}\label{sec:moeb}

In Subsection \ref{basic} we recall some basics of M\"obius photogrammetry, 
which are needed for the construction of the three possible pentapod designs (up to affinities 
of the planar platform and the planar base) given in Subsection \ref{10pd}.

\subsection{Basics}\label{basic}

First of all we need the notation of a so-called M\"obius transformation $\gamma$ of the plane. 
If we combine the planar Cartesian coordinates $(x,y)$ to a 
complex number $z:=x+iy$, then $\gamma(z)$ can be defined as a  rational function of the form 
\begin{equation}
\gamma:\quad z\mapsto \frac{z_1z+z_2}{z_3z+z_4},
\end{equation}
with complex numbers $z_1, \ldots, z_4$ satisfying $z_1z_4 - z_2z_3 \neq 0$. Therefore 
M\"obius transformations can be seen as the projective transformations of the complex 
projective line $\PP^1_{\CC}$. 

We identify by the mapping $\iota$ the unit-sphere $S^2$ of the Euclidean 3-space $\RR^3$ with an algebraic curve  $C:=\left\{x^2+y^2+z^2=0\right\}$ 
in $\PP^2_{\CC}$. In detail this identification works as follows: Let $\Vkt u\in S^2$ with $\Vkt u=(u_1,u_2,u_3)$. Then determine $\Vkt v,\Vkt w\in S^2$ 
with  $\Vkt v=(v_1,v_2,v_3)$ and  $\Vkt w=(w_1,w_2,w_3)$ in a way that $\Vkt u,\Vkt v, \Vkt w$ determine a right-handed basis of $\RR^3$. Then the map 
$\iota :\, S^2 \rightarrow  \PP^2_{\CC}$ is given by:
\begin{equation}
\iota: (u_1,u_2,u_3) \mapsto (v_1+iw_1:v_2+iw_2:v_3+iw_3)
\end{equation}
as a different choice of $\Vkt v,\Vkt w\in S^2$ yields to the same point in $\PP^2_{\CC}$.

By denoting the vector $(\go M_1,\ldots ,\go M_5)$ of five points in $\RR^3$ by $\mathfrak{M}$,  
the orthogonal parallel projection $\pi$ of $\mathfrak M$ along the direction 
associated with $c\in C$ is given by $\pi_c(\mathfrak{M})$. By writing the planar Cartesian coordinates of each projected point 
as a complex number we get $\pi_c(\mathfrak{M})\in (\PP^1_{\CC})^5$.

\begin{rmk}\label{affine}
Assume that $\go M_1,\ldots ,\go M_5$ is known to be coplanar, the 5-tuple can be reconstructed from 
$\pi_c(\mathfrak{M})$ only up to affinity, as the orientation of the carrier plane of the 5 points with 
respect to $\iota^{-1}(c)$ is not known. This also corrects \cite{erratum}, where "similarity" is written 
instead of "affinity".
\hfill $\diamond$
\end{rmk}

The equivalence class under the action of the 
M\"obius group $\Gamma$ on $\pi_c(\mathfrak{M})$ is the so-called 
M\"obius picture $[\pi_c(\mathfrak{M})]_{\Gamma}$ of  $\mathfrak M$ along the direction associated with $c\in C$.  

The set of all these equivalence classes $[(\PP^1_{\CC})^5]_{\Gamma}$ 
can be viewed as a quintic surface $P_5\in\PP_{\CC}^5$ known as Del Pezzo surface. For 
$\pi_c(\mathfrak{M})$ with coordinates $(x_1+iy_1,\ldots ,x_5+iy_5)\in (\PP^1_{\CC})^5$ the corresponding point of the 
Del Pezzo surface is defined as $(\phi_0:\phi_1:\phi_2:\phi_3:\phi_4:\phi_5)$ with: 
\begin{equation}\label{phi}
\begin{split}
\phi_0&:=D_{12}D_{23}D_{34}D_{45}D_{15},\\
\phi_1&:=D_{12}D_{25}D_{15}D_{34}D_{34},\\
\phi_2&:=D_{12}D_{23}D_{13}D_{45}D_{45},\\
\phi_3&:=D_{23}D_{34}D_{24}D_{15}D_{15},\\
\phi_4&:=D_{34}D_{45}D_{35}D_{12}D_{12},\\
\phi_5&:=D_{14}D_{45}D_{15}D_{23}D_{23},
\end{split}
\end{equation}
and $D_{ij}:=x_iy_j-x_jy_i$. 
For details of this construction 
of $P_5$ we refer to \cite[Section 3.1]{gns2}, but it is important to note that $P_5$ carries 10 lines 
$L_{ij}$ corresponding to equivalence classes for which the projection of the $i$th and the $j$th point coincide 
($\Leftrightarrow$ $D_{ij}=0$) for pairwise distinct $i,j\in\left\{1, \ldots , 5\right\}$. 

We are interested in the set of M\"obius pictures of $\mathfrak M$ under all $c\in C$. 
By applying the so-called photographic map $f_{\mathfrak M}$ of $\mathfrak M$ given by
\begin{equation}
f_{\mathfrak M}: C\rightarrow P_5 \quad\text{with}\quad
c\mapsto  [\pi_c(\mathfrak{M})]_{\Gamma}
\end{equation}
we can compute the so-called profile $p_{\mathfrak M}$ of $\mathfrak M$ as the Zariski closure of $f_{\mathfrak M}(C)$; i.e.\ 
$ZarClo(f_{\mathfrak M}(C))$. 
Note that the profile  is a curve on $P_5$. 
 
According to \cite[Remark 3.5]{gns2} the M\"obius picture cannot be defined for those values $c\in C$, for which the 
associated directions are parallel to three collinear points of $\mathfrak M$, as in this case all five $\phi_i$'s are equal 
to zero. In our case two such directions exist, 
which are parallel to the carrier line of $\go M_1,\go M_2,\go M_3$ (i.e. the metallic direction $m$) and 
$\go M_3,\go M_4,\go M_5$ (i.e. the blue direction $b$), respectively (cf.\ Fig.\ \ref{fig1}). 
However, we can extend $f_{\mathfrak M}$ also to these directions by canceling out the 
common vanishing factor. 
For our given base (cf.\ Fig.\ \ref{fig1}) this common factor is $D_{12}=D_{13}=D_{23}$ (resp.\ $D_{34}=D_{35}=D_{45}$) for the 
metallic (resp.\ blue) direction, thus Eq.\ (\ref{phi}) yields $(0:D_{25}D_{34}:0:D_{24}D_{15}:0:0)$ 
(resp.\ $(0:0:0:D_{24}D_{15}:0:D_{14}D_{23}))$. Therefore the $m$-direction (resp.\ $b$-direction) is mapped on a point 
of $L_{45}$ (resp.\ $L_{12}$); i.e. 
\begin{equation}\label{eq:base}
[\pi_m(\mathfrak{M})]_{\Gamma}\in L_{45}, \quad
[\pi_b(\mathfrak{M})]_{\Gamma}\in L_{12}. 
\end{equation}
\begin{figure}[h!]
\begin{center}  
 \begin{overpic}  
    [width=70mm]{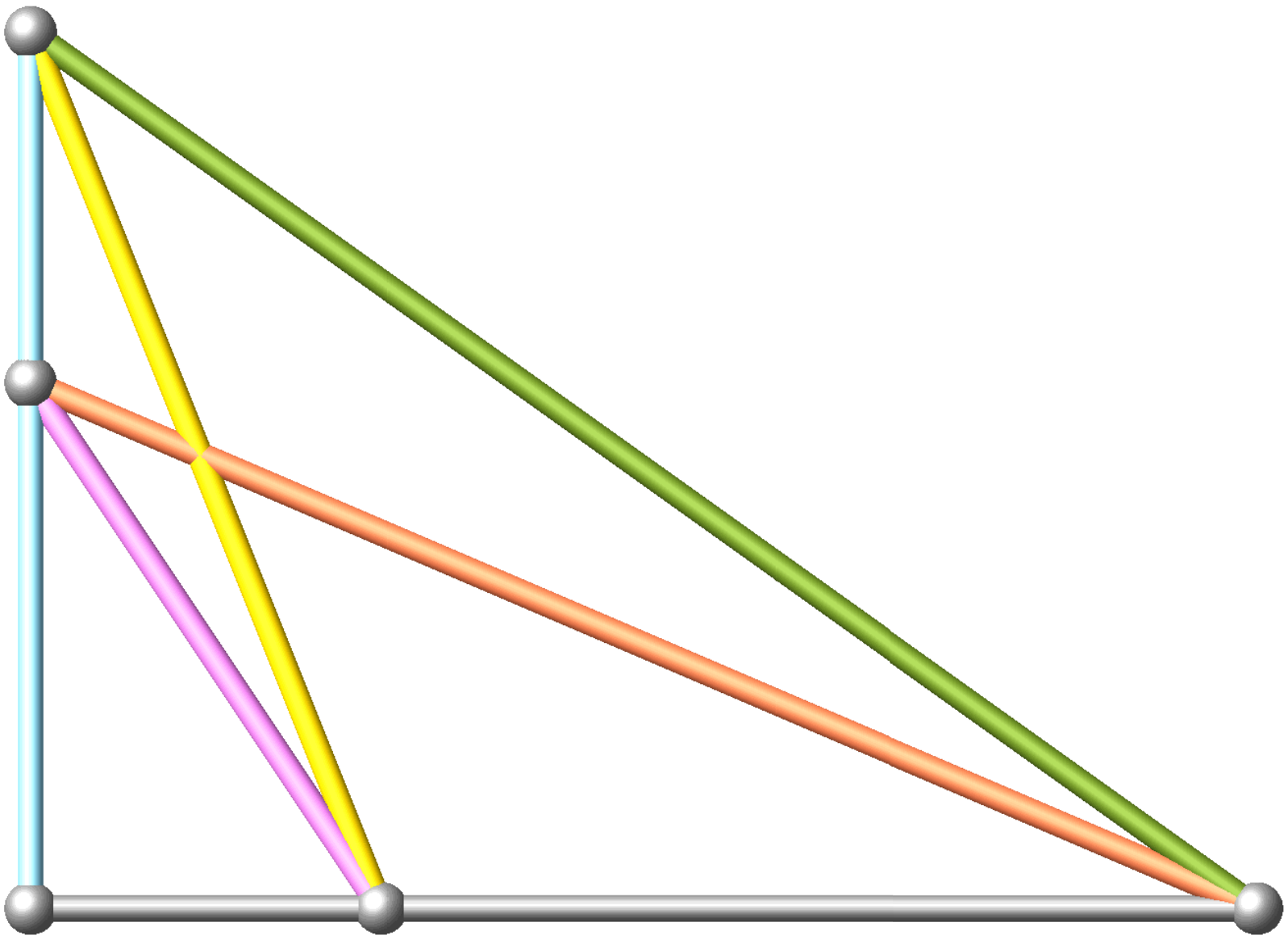}
	\put(45,43){\begin{turn}{-34}green ($g$)\end{turn}}	
	\put(40,21){\begin{turn}{-23}orange ($o$)\end{turn}}	
	\put(47,5){metallic ($m$)}
	\put(22,30){\begin{turn}{-70}yellow ($y$)\end{turn}}	
	\put(9,23){\begin{turn}{-60}pink ($p$)\end{turn}}	
	\put(-4,30){\begin{turn}{-90}blue ($b$)\end{turn}}	
	\put(-7,1){$\go M_3$}
	\put(32,5){$\go M_1$}
	\put(101,1){$\go M_2$}
	\put(-7,39){$\go M_4$}
	\put(-7,69){$\go M_5$}
  \end{overpic} 
\end{center}
\caption{
The photographic map sends any direction vector parallel to a line through 2 (but not 3) of the points $\go M_i$, $\go M_j$ 
to the unique point in the M\"obius picture on the line $L_{ij}$ of the quintic surface $P_5$. In the base configuration above,  
green ($g$) is sent to $L_{25}$, orange ($o$) is sent to $L_{24}$, yellow ($y$) is sent to $L_{15}$ and pink ($p$) is sent to $L_{14}$. 
It is not clear whether the directions blue ($b$) and metallic ($m$) are being sent; later, we will show that 
$b$ is sent to  $L_{12}$ and $m$ is sent to $L_{45}$.
}
  \label{fig1} 
\end{figure}
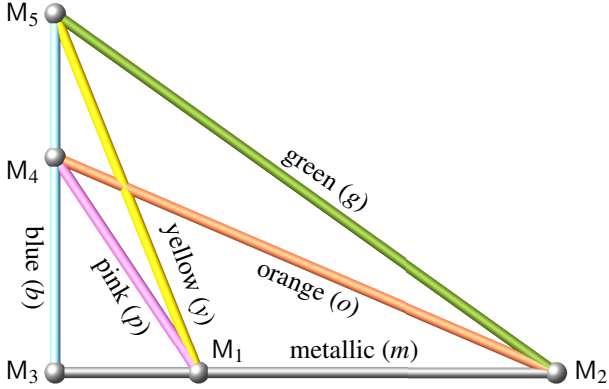

\subsection{Three possible designs}\label{10pd}

It is known (see \cite[Section 4]{gns2}) that for a pentapod  with mobility 2, which belongs to item (e) of 
Theorem \ref{thm1a}, the profiles $p_{\mathfrak M}$ and $p_{\mathfrak m}$ have to coincide, where  $\mathfrak{m}$ denotes the 
vector of five points $(\go m_1,\ldots ,\go m_5)$. 
As a consequence there has to be a one-to-one correspondence 
between $p_{\mathfrak M}$ and $p_{\mathfrak m}$, which is used to 
reconstruct in three ways $\mathfrak m$ (up to affinity; cf.\ Remark \ref{affine}), 
under the assumption that  $\mathfrak M$ is given (cf.\ Fig.\ \ref{fig1}). 

\begin{assumption}\label{ass:1}
W.l.o.g.\ we can assume that the reconstruction $\mathfrak m$ 
is affinely transformed in a way that the M\"obius pictures (and their extension in the case of three collinear points) 
of $\mathfrak{m}$ and $\mathfrak{M}$ with respect to any direction $c$ are identical.
\end{assumption}

First of all we have to distinguish the following three cases, which are 
implied by the three possible collinearity configurations stated in (e):
\begin{enumerate}[1.]
\item
$i=2$: W.l.o.g.\ we can set $j=4$ and $k=5$. 

One can select $\go m_1$ arbitrarily. As $L_{14}\cap p_{\mathfrak m}$ has to coincide with 
$L_{14}\cap p_{\mathfrak M}$ the line $\go m_1\go m_4$ has to be parallel to $\go M_1\go M_4$. Now we can select any point ($\neq \go m_1$) on the 
parallel line to $\go M_1\go M_4$ through $\go m_1$ as $\go m_4$. 
The direction of $\go m_1\go m_2$ is not 
uniquely determined as the line $\go M_1\go M_2$ also contains the point $\go M_3$. Due to the one-to-one correspondence 
between the two profiles, 
$L_{12}\cap p_{\mathfrak m}$ has to correspond with one of the two points on $p_{\mathfrak M}$, which do not admit a M\"obius picture. 
Therefore there are the following two possibilities:   
	\begin{enumerate}[(a)]
	\item
	$\go m_1\go m_2$ is parallel to $\go M_1\go M_2$: As a consequence $\go m_4\go m_5$ has to be parallel to $\go M_4\go M_5$. 
	
	Moreover as $L_{24}\cap p_{\mathfrak m}$ has to coincide with $L_{24}\cap p_{\mathfrak M}$ the line $\go m_2\go m_4$ has to be parallel to $\go M_2\go M_4$.
	Therefore we get $\go m_2$ as the intersection point of a parallel line to $\go M_1\go M_2$ through $\go m_1$ and a parallel line to 
	$\go M_2\go M_4$ through $\go m_4$. 
	
	In the same way $L_{15}\cap p_{\mathfrak m}$ has to coincide with $L_{15}\cap p_{\mathfrak M}$ and therefore 
	$\go m_5$ can be obtained as the intersection point of a parallel line to $\go M_4\go M_5$ through $\go m_4$ and a parallel line to 
	$\go M_1\go M_5$ through $\go m_1$.
	
	Although all points are reconstructed, we have to check if the last remaining condition is fulfilled, namely if 
	$\go m_2\go m_5$ is parallel to $\go M_2\go M_5$. As this can easily be verified, we get  
	reconstruction 1 illustrated in Fig.\ \ref{fig2}(left), which is in fact identical with $\mathfrak{M}$ (cf.\ Fig.\ \ref{fig1}).
\begin{figure}[h!]
\begin{center}  
 \begin{overpic}  
    [width=70mm]{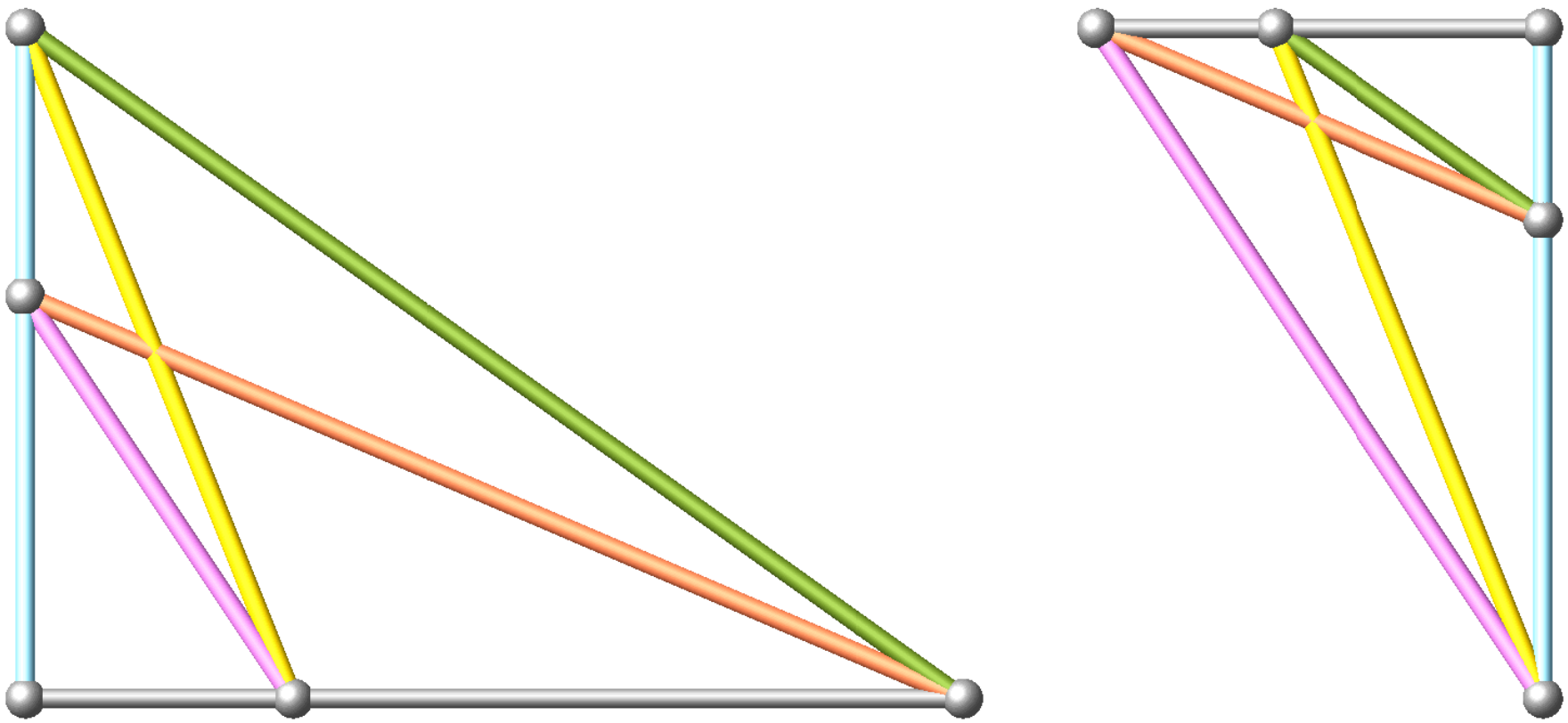}
	\put(-6,1){$\go m_3$}
	\put(20,3.5){$\go m_1$}
	\put(64,1){$\go m_2$}
	\put(-6,25){$\go m_4$}
	\put(-6,44){$\go m_5$}
	\put(101,1){$\go m_1$}	
	\put(101,31){$\go m_2$}	
	\put(101,44){$\go m_3$}	
	\put(62,44){$\go m_4$}
	\put(79,47){$\go m_5$}	
  \end{overpic} 
\end{center}
\caption{
The two possible reconstructions of the platform configuration from the M\"obius picture, under the additional 
assumption that $\go m_3$ is the intersection of lines $\go m_1\go m_2$ and $\go m_3\go m_4$. 
Note that the line $\go m_2\go m_5$ must have 
direction $g$, the line $\go m_2\go m_4$ must have direction $o$, the line $\go m_1\go m_5$ must have direction $y$ and 
the line $\go m_1\go m_4$ must have direction $p$. The left configuration coincides with the base configuration. We will see 
later that the right configuration is not compatible because the lines through $\go m_1,\go m_2,\go m_3$ and 
$\go m_3,\go m_4,\go m_5$, respectively,  do not have the correct directions. 
}
  \label{fig2} 
\end{figure}
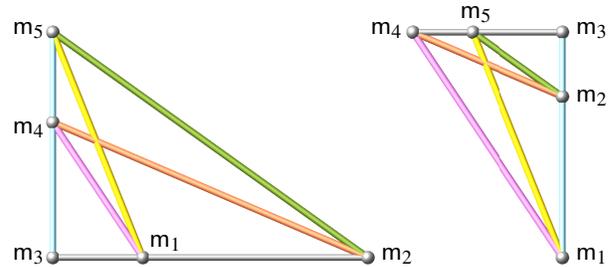  
	\item
	$\go m_1\go m_2$ is parallel to $\go M_4\go M_5$: As a consequence $\go m_4\go m_5$ has to be parallel to $\go M_1\go M_2$. 
	
	Analogous arguments as in the above case with respect to the swapped directions yield a further candidate platform $\mathfrak{m}$ 
	illustrated in Fig.\ \ref{fig2}(right). 
	Calculation of the M\"obius picture of $\mathfrak{m}$ with respect to the directions $m$ and $b$ according to Eq.\ (\ref{phi}) 
	shows
	\begin{equation}\label{falsch}
	[\pi_m(\mathfrak{m})]_{\Gamma}\in L_{12}, \quad
	[\pi_b(\mathfrak{m})]_{\Gamma}\in L_{45}.  
	\end{equation}
	Due to Eq.\ (\ref{eq:base}), $\mathfrak{m}$ and $\mathfrak{M}$ do not have the same M\"obius picture with respect to the directions 
	$m$ and $b$; 	a contradiction.
	\end{enumerate}	
\item	
$i=5$: W.l.o.g.\ we can set $j=2$ and $k=4$. 

For the same reasons as in item 1 we can select $\go m_1$ arbitrarily and can choose any point ($\neq \go m_1$) on the 
parallel line to $\go M_1\go M_4$ through $\go m_1$ as $\go m_4$. Moreover the following two subcases can also be reasoned analogously to item 1: 
	\begin{enumerate}[(a)]
	\item
	$\go m_1\go m_2$ is parallel to $\go M_1\go M_2$: As a consequence $\go m_4\go m_5$ has to be parallel to $\go M_4\go M_5$. 

		In this case we also get Eq.\ (\ref{falsch}), which implies the same contradiction as in case 1(b).
	
	\item
	$\go m_1\go m_2$ is parallel to $\go M_4\go M_5$: As a consequence $\go m_4\go m_5$ has to be parallel to $\go M_1\go M_2$.
	
	As now the line $\go m_2\go m_4$ also contains the point $\go m_3$ the corresponding direction does not admit a 
	M\"obius picture. Due to the one-to-one correspondence between $p_{\mathfrak M}$ and $p_{\mathfrak m}$ again two cases have 
	to be distinguished:
		\begin{enumerate}[i.]
		\item
		$\go m_2\go m_4$ is parallel to $\go M_1,\go M_5$: As a consequence $\go m_1\go m_5$ has to be parallel to $\go M_2\go M_4$.
		
		Therefore $\go m_2$ can be obtained as the intersection point of the parallel line to  $\go M_1\go M_2$ through $\go m_1$ and
		the parallel line to $\go M_1\go M_5$ through $\go m_4$. Moreover $\go m_5$ equals the intersection point of the parallel 
		line to $\go M_4\go M_5$ through $\go m_4$ and the parallel line to $\go M_2\go M_4$ through $\go m_1$.
		
		Although all points are reconstructed, we have to check again if the last remaining condition is fulfilled, namely if 
	  $\go m_2\go m_5$ is parallel to $\go M_2\go M_5$. As this can easily be verified, we get reconstruction 2
		illustrated in Fig.\ \ref{fig4}(right). 
\begin{figure}[h!]
\begin{center}  
 \begin{overpic}  
    [width=70mm]{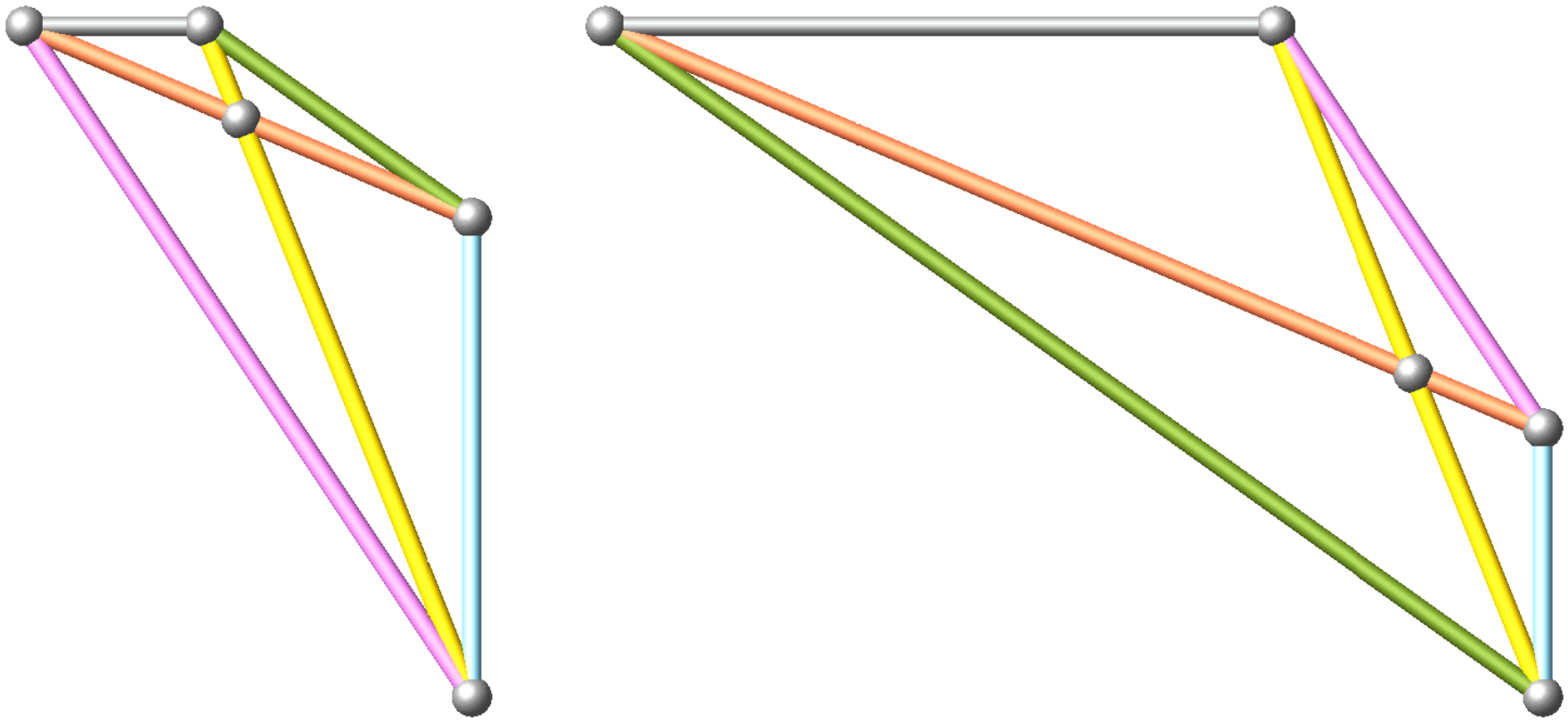}
	\put(9.5,34){$\go m_3$}
	\put(33,1){$\go m_1$}
	\put(33,30.5){$\go m_2$}
	\put(15.5,44){$\go m_5$}
	\put(-6,44){$\go m_4$}
	\put(101,1){$\go m_2$}	
	\put(101,17){$\go m_1$}	
	\put(83.5,18.5){$\go m_3$}	
	\put(30.5,44){$\go m_5$}
	\put(84.5,44){$\go m_4$}	
  \end{overpic} 
\end{center}
\caption{
The two possible reconstructions of the platform configuration from the M\"obius picture, under the additional 
assumption that $\go m_3$ is the intersection of lines $\go m_1\go m_5$ and $\go m_2\go m_4$. Here the directions of lines $\go m_1\go m_2$, 
$\go m_4\go m_5$, $\go m_1\go m_4$ and $\go m_2\go m_5$ are fixed to $b$, $m$, $p$ and $g$, respectively. We will later see that the left configuration 
is not compatible. The right configuration leads to a Duporcq pentapod.}
  \label{fig4} 
\end{figure}
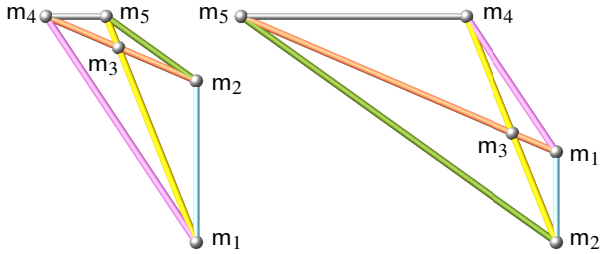  
		\item
		$\go m_2\go m_4$ is parallel to $\go M_2\go M_4$: As a consequence $\go m_1\go m_5$ has to be parallel to $\go M_1\go M_5$. 
		
		Analogous considerations as in item 2(b)i yields the candidate platform illustrated in Fig.\ \ref{fig4}(left). 
		Now the calculation of the 
		M\"obius picture of this candidate with respect to the orange direction $o$ yields 
		$(0:0:D_{13}D_{45}:0:D_{35}D_{12}:0)$. Therefore we have $[\pi_o(\mathfrak{m})]_{\Gamma}\in L_{15}$, 
		which contradicts $[\pi_o(\mathfrak{M})]_{\Gamma}\in L_{24}$, thus we have no valid reconstruction. 
		\end{enumerate}
	\end{enumerate}
\item
$i=4$: W.l.o.g.\ we can set $j=2$ and $k=5$.

The discussion of cases is exactly the same as in item 2 if one  
exchanges the indices 4 and 5. The resulting reconstruction 3 as well as the 
corresponding non-valid candidate platform are illustrated in Fig.\ \ref{fig6}.
\end{enumerate}
\begin{figure}[h!]
\begin{center}  
 \begin{overpic}  
    [width=70mm]{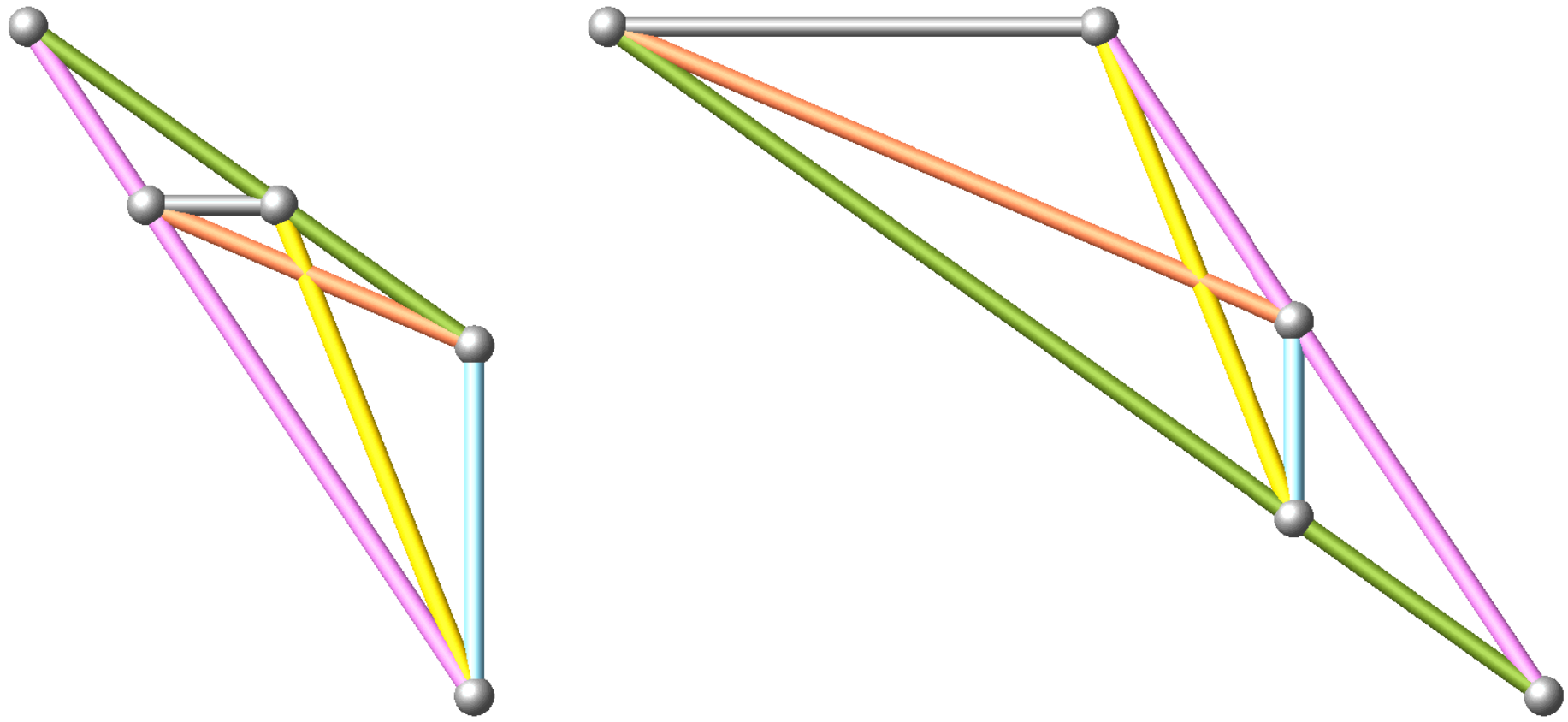}
	\put(2,31){$\go m_4$}
	\put(33,1){$\go m_1$}
	\put(33,23){$\go m_2$}
	\put(20,33){$\go m_5$}
	\put(-6,44){$\go m_3$}
	\put(101,1){$\go m_3$}	
	\put(75.5,10){$\go m_1$}	
	\put(85,25){$\go m_2$}	
	\put(30.5,44){$\go m_4$}
	\put(72,44){$\go m_5$}	
  \end{overpic} 
\end{center}
\caption{
The two possible reconstructions of the platform configuration from the M\"obius picture, under the additional 
assumption that $\go m_3$ is the intersection of lines $\go m_1\go m_4$ and $\go m_2\go m_5$. Here the directions of lines $\go m_1\go m_2$, 
$\go m_4\go m_5$, $\go m_1\go m_5$ and $\go m_2\go m_4$ are fixed to $b$, $m$, $y$, and $o$, respectively. The left configuration 
is not compatible. The right configuration leads to a Duporcq pentapod.
}
  \label{fig6} 
\end{figure}
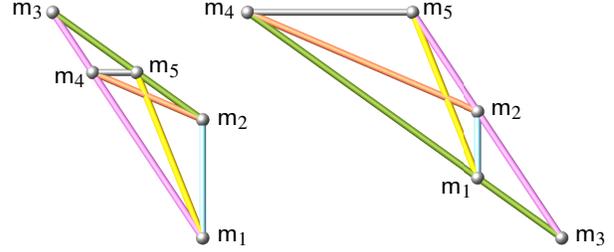  
Moreover for the $i$th reconstruction ($i=1,2,3$) there exists an affine relation $\kappa_i$
between the set $\left\{\go M_1,\go M_2,\go M_4,\go M_5\right\}$ and the 
set  $\left\{\go m_1,\go m_2,\go m_4,\go m_5\right\}$. In detail these affine mappings $\kappa_i$
are given by:
\begin{align*}
&\kappa_{1}:&\, &\go M_1\mapsto \go m_1 &\,   &\go M_2\mapsto \go m_2 &\,  &\go M_4\mapsto \go m_4 &\,  &\go M_5\mapsto \go m_5 \\
&\kappa_{2}:&\, &\go M_1\mapsto \go m_4 &\,   &\go M_2\mapsto \go m_5 &\,  &\go M_4\mapsto \go m_1 &\,  &\go M_5\mapsto \go m_2 \\
&\kappa_{3}:&\, &\go M_1\mapsto \go m_5 &\,   &\go M_2\mapsto \go m_4 &\,  &\go M_4\mapsto \go m_2 &\,  &\go M_5\mapsto \go m_1 
\end{align*}
For all three cases the validity of these affine mappings can be proven by 
direct computation. Moreover it should be noted that $\kappa_1$ 
maps $\go M_3\mapsto \go m_3$ in addition. 
As a consequence the pentapod design resulting from reconstruction 1 is a planar affine pentapod 
belonging to item (b) of Theorem \ref{thm1a}, which was already discussed in \cite{asme_ns}. 
Therefore we remain only with reconstruction 2 and 3. 

\begin{rmk}\label{rmk:dup}
Note that the Duporcq pentapods fit with reconstruction 2 and 3 for the following reason: 
Assumed the base (cf.\ Fig.\ \ref{fig1}) is given, then two lines of the complete quadrilateral through the points 
$\go M_1,\go M_2,\go M_4,\go M_5$ are already determined by the collinearity of the triples 
$\go M_1,\go M_2,\go M_3$ and $\go M_3,\go M_4,\go M_5$, respectively. 
Therefore the quadrilateral is completed either by the lines 
$\go M_1\go M_5$ and $\go M_2\go M_4$, which corresponds with reconstruction 2, 
or by the lines $\go M_1\go M_4$ and $\go M_2\go M_5$, which corresponds with reconstruction 3.  
\hfill $\diamond$
\end{rmk}


\section{Bond Theory}\label{bonds}

In this section we shortly repeat two different approaches for defining so-called bonds. 
The first one discussed in Subsection \ref{studybonds} is based on the Study parametrization of 
$SE(3)$ in contrast to the one presented in Subsection \ref{confbonds}, which uses the so-called 
conformal embedding of $SE(3)$.  In Section \ref{vs} a 
relation between the bonds based on these different embeddings is given.

\subsection{Bonds based on the Study Embedding of SE(3)}\label{studybonds}

We denote the eight homogenous Study parameters by $(e_0:e_1:e_2:e_3:f_0:f_1:f_2:f_3)$, where 
the first four homogeneous coordinates $(e_0:e_1:e_2:e_3)$ are the so-called Euler parameters.
Now, all real points of the Study parameter space $\PP^7$, 
which are located on the so-called Study quadric $S:\,\sum_{i=0}^3e_if_i=0$, 
correspond to an Euclidean displacement, with exception of the 3-dimensional subspace $e_0=e_1=e_2=e_3=0$, 
as its points cannot fulfill the condition $N\neq 0$ with $N=e_0^2+e_1^2+e_2^2+e_3^2$. 
All points of the complex extension $\PP^7_{\CC}$ of $\PP^7$, which cannot fulfill this normalizing condition,  
are located on the so-called exceptional quadric $N=0$. 

By using the Study parametrization of Euclidean displacements, the condition that the point $\go m_i$ is located 
on a sphere centered in $\go M_i$ with radius $R_i$ is a quadratic homogeneous equation according to Husty \cite{husty}.  
For the explicit formula of this so-called sphere condition $Q_i$ we refer to \cite[Eq.\ (2)]{asme_ns}. 
Now the solution for the direct kinematics over $\CC$ of a pentapod can be written as the
algebraic variety $V$ of the ideal spanned by $S,Q_1,\ldots ,Q_5,N=1$. 
In the case of pentapods with mobility 2 the variety $V$ is $2$-dimensional.

We consider the algebraic motion of the pentapod, which is defined as the set of points on the
Study quadric determined by the constraints; i.e., the  common points of  
the six quadrics $S,Q_1,\ldots ,Q_5$. 
Now the points of the algebraic motion with $N\neq 0$ equal the kinematic image of the algebraic variety $V$. 
But we can also consider the set $\mathcal{B}$ of points of the algebraic motion, which belong to the exceptional quadric $N=0$. 
For an exact mathematical definition of these so-called bonds we refer to \cite[Definition 1]{asme_ns}. 
In the case of pentapods with mobility 2 the set $\mathcal{B}$ is of dimension $1$; i.e., a {\it bonding curve}.

We use the following approach for the computation of bonds: 
In a first step we project the algebraic motion of the pentapod 
into the Euler parameter space $\PP^3_{\CC}$ by the elimination of $f_0,\ldots ,f_3$. 
This projection is denoted by $\varsigma$. 
In a second step we determine the set $\mathcal{B}_{\varsigma}$ of projected bonds as 
those points of the projected point set $\varsigma(V)$, which are located on the quadric $N=0$; i.e., 
\begin{equation}\label{def:projbond}
\mathcal{B}_{\varsigma}:= ZarClo\left(\varsigma\left(V\right)\right) \cap \left\{
(e_0:\ldots :e_3)\in \PP^3_{\CC}\,\, | \,\,  N=0 
\right\}.
\end{equation}

\subsection{Bonds based on the Conformal Embedding of SE(3)}\label{confbonds}

As shown in \cite[Section 2.1]{gns1}, it is 
possible to construct a projective compactification $X$ in $\PP^{16}_{\CC}$ for 
the complexification $SE(3)_{\CC}$ of the group $SE(3)$ in a way that the sphere condition 
is linear in the coordinates of $\PP^{16}_{\CC}$. 
The map $SE(3) \hookrightarrow \PP^{16}_{\CC}$ is the so-called \emph{conformal 
embedding} of $SE(3)$ and $X$ is a projective variety of dimension $6$ and degree $40$. 

Now the five linear sphere conditions  
determine a linear subspace $F \subseteq \PP^{16}_{\CC}$ of codimension $5$.
The intersection $K = X \cap F$ is defined to be the \emph{complex 
configuration set} of the pentapod. 

It is also known that $X$ can be written as the 
disjoint union $SE(3)_{\CC} \cup B_X$, where the so-called boundary 
${B}_X$ is obtained as the intersection of $X$ and a hyperplane $H$. 
Moreover the boundary can be decomposed into the following $5$ subsets:
\begin{description}
\item[Vertex:] This is the only real point in ${B}_X$, a singular point with 
multiplicity~$20$; it is never contained in $K$.
\item[Collinearity points:] If $K$ contains such a point, then either the
platform points or the base points are collinear. 
\item[Similarity points:] If $K$ contains such a point, then there are normal 
projections of platform and base to a plane such that the images are similar.
\item[Inversion points:] If $K$ contains such a point, then there are normal
projections of platform and base to a plane such that the images are related
by an inversion.
\item[Butterfly points:] If $K$ contains such a point, then there are two lines,
one in the base and one in the platform, such that any leg has either its base 
point on the base line or its platform point on the platform line.
\end{description}

Now the set of bonds $\mathcal{B}_K$ is obtained as the intersection of 
$K$ and the  boundary ${B}_X$. 
Moreover it should be mentioned that the intersection multiplicity of 
$K$ and $H$ is at least $2$ in each bond. 
Note that for pentapods with mobility 2, the bondset $\mathcal{B}_K$ is 1-dimensional.

\subsection{Relation between Bonds based on different Embeddings}\label{vs}
 
If $\rho \colon SE(3) \longrightarrow SO(3)$ is the map sending a direct isometry 
to its rotational part, then there exists a linear projection $\xi\colon \PP^{16}_{\CC} \dashrightarrow 
\PP^{9}_{\CC}$ such that the following diagram is commutative (cf.\ \cite[Section 1]{liaison}):
\begin{equation}
\label{diagram:compactifications}
  \begin{array}{c}
  \xymatrix{ SE(3) \ar[dr]\ar[rr]\ar[dd]_{\rho}  & & X \subseteq\PP^{16}_{\CC} \ar@{-->}[dd]_{\xi} 
	 \\ 
	& S\subseteq\PP^7_{\CC} \ar@{-->}[d]_{\varsigma} & \\
SO(3) \ar[r] & \PP^3_{\CC} \ar[r]^-{v_{3,2}} & V_{3,2} \subseteq \PP^9_{\CC} } 
  \end{array}
\end{equation}
where $v_{3,2}$ is the Veronese embedding of~$\PP^{3}_{\CC}$ and 
$V_{3,2}$ is its image in~$\PP^9_{\CC}$. 
The center of $\xi$ is the linear space spanned by 
similarity points, which contains also the collinearity points and the vertex. 

\begin{lem}\label{lem:0}
For reconstruction 2 and 3 the complex configuration set $K$ does not contain collinearity bonds, 
but four butterfly bonds and one similarity bond.   
\end{lem}

\begin{proof}
The numbers of collinearity and butterfly bonds are trivial. The reasoning for the existence of 
exactly one similarity bond is as follows (cf.\ Fig.\ \ref{figsim}): 

As $\go M_1,\go M_2,\go M_3$ are collinear the ratio $TV(\go M_1,\go M_2,\go M_3)$ remains 
constant under parallel projections (with projection directions not parallel to the carrier line of the collinear points). 
Therefore one can construct the point $\go m_3^{\prime}$ on 
the line $\go m_1\go m_2$ such that $TV(\go m_1,\go m_2,\go m_3^{\prime})=TV(\go M_1,\go M_2,\go M_3)$ holds.
In the same way one can construct the point $\go m_3^{\prime\prime}$ on 
the line $\go m_4\go m_5$ such that 
$TV(\go m_3^{\prime\prime},\go m_4,\go m_5)=TV(\go M_3,\go M_4,\go M_5)$ holds. 
It can be checked by direct computations that $\go m_3,\go m_3^{\prime},\go m_3^{\prime\prime}$ 
are located on a line $\go g$,
which gives the direction of the projection direction of the platform. 

The reverse construction from the platform to the base yields the points 
$\go M_3,\go M_3^{\prime},\go M_3^{\prime\prime}$ located on a line $\go G$,  
which gives the direction of the projection of the base. 

Moreover $\go g$ and $\go G$ have to be parallel due to Assumption \ref{ass:1}, 
which can also be checked by straightforward computations. \hfill $\BewEnde$
\end{proof}

\begin{figure}[h!]
$\phm$\hfill
 \begin{overpic}  
    [width=80mm]{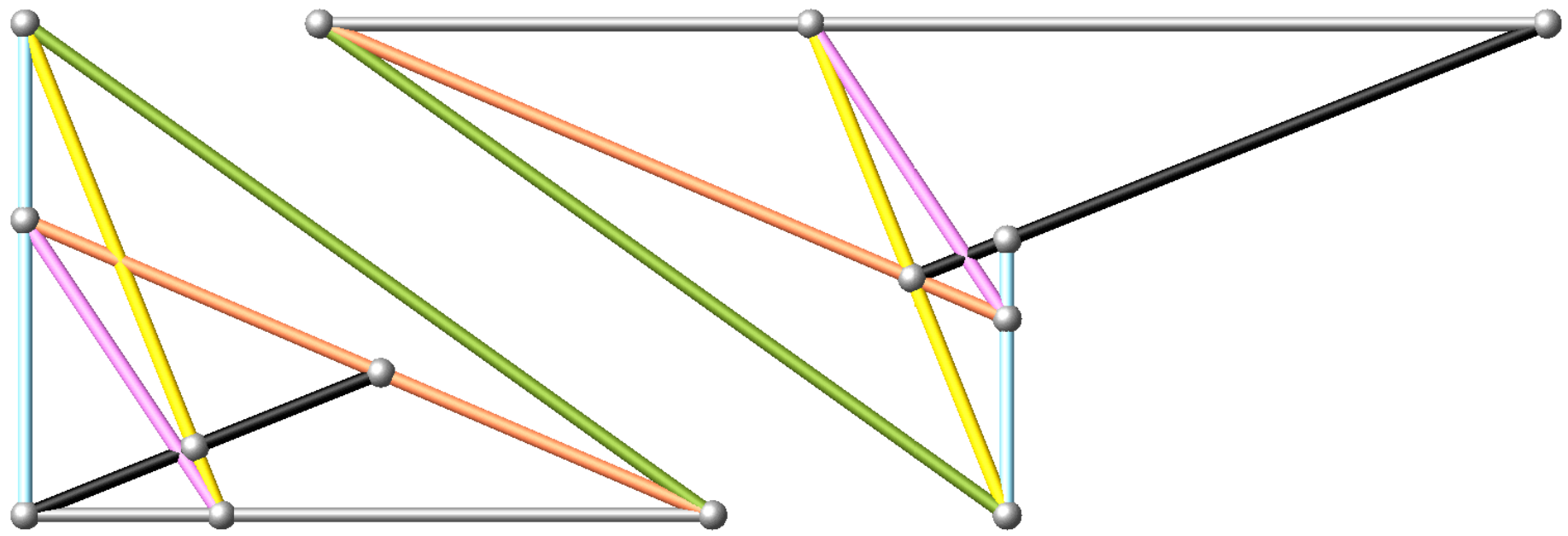}
		\begin{footnotesize}
	\put(-4.5,1){$\go M_3$}
	\put(15.5,2.7){$\go M_1$}
	\put(47,1){$\go M_2$}
	\put(-4.5,18){$\go M_4$}
	\put(-4.5,30){$\go M_5$}
	\put(11.7,9){$\go M_3^{\prime}$}
	\put(21.5,5.5){$\go M_3^{\prime\prime}$}		
	\put(5,4.5){$\go G$}
	\put(66,1){$\go m_2$}	
	\put(66,12){$\go m_1$}	
	\put(53,13.5){$\go m_3$}	
	\put(62,21.5){$\go m_3^{\prime}$}	
	\put(95,28){$\go m_3^{\prime\prime}$}	
	\put(16,29.5){$\go m_5$}
	\put(47,29.5){$\go m_4$}
	\put(77,26.5){$\go g$}
	\end{footnotesize}
	\end{overpic}
\caption{
Projection along the black direction leads to one-dimensional configurations that are similar. This shows that 
the pentapod with such base/platform configuration has a similarity bond.
}
  \label{figsim} 
\end{figure}
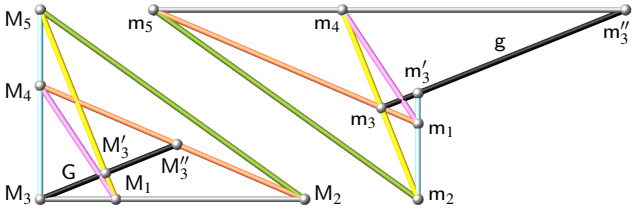  

\begin{cor} \label{cor:0} 
If reconstruction 2 or 3 has mobility 2, then $K_e$ has to be a surface. 
\end{cor}

\begin{proof}
First of all we want to recall the known characterization for a pure translational self-motion 
(according to \cite[Theorem 2]{nawratil_bond} under consideration of 
\cite[Footnote 4]{asme_ns}): 
A pentapod possesses a pure translational self-motion, 
if and only if the platform can be rotated about the center $\go m_1=\go M_1$ into a pose, where the vectors 
$\overrightarrow{\go M_i\go m_i}$ for $i=2,\ldots ,5$ 
fulfill the condition  $rk(\overrightarrow{\go M_2\go m_2},\ldots ,\overrightarrow{\go M_5\go m_5})\leq 1$. Note 
that this implies the existence of a similarity bond. 
Moreover all $1$-dimensional self-motions are circular translations in planes orthogonal to the parallel vectors 
$\overrightarrow{\go M_i\go m_i}$ for $i=2,\ldots ,5$.

If $K_e$ is a point, then the orientation during the self-motion is fixed and we can only obtain a 2-dimensional 
translational self-motion. It is well-known \cite{nawratil_bond} that in this case the platform and the base have to be directly congruent. 

If $K_e$ is a curve, then for each corresponding platform orientation a $1$-dimensional translational sub-self-motion 
has to exist. This implies a $1$-dimensional set of similarity bonds, which contradicts Lemma \ref{lem:0}.  \hfill $\BewEnde$
\end{proof}

\begin{lem} \label{lem:1}
Assume that $K$ is a surface. Assume that the projection $v_{3,2}^{-1}\circ\xi$: 
$K\mapsto K_e$ is birational.
Assume that the pod has infinitely many inversion bonds.
Then the intersection of $K_e$ and the exceptional quadric $N=0$
has a curve of multiplicity bigger than 1.
\end{lem}

\begin{proof}
The mobility surface $K$ has a tangential intersection with the boundary hyperplane $H$
at all inversion bonds. Since the projection $K\to K_e$ is birational, the projection 
$\xi:X\subset\PP^{16}\dashrightarrow\PP^9$ is locally an isomorphism for almost all inversion points.
Hence the image of $K$ intersects the image of the hyperplane $H$ tangentially at almost all
images of inversion points. This is equivalent to saying that $K_e$ intersects the exceptional quadric $N=0$ 
at almost all images of inversion points. The closure of these points would be
the curve with intersection multiplicity bigger than 1. \hfill $\BewEnde$
\end{proof}

\begin{cor} \label{cor:1}
If in addition to the above assumptions $K_e$ is a quadric surface, then the intersection
is totally tangential along an irreducible quadric.
\end{cor}

\begin{proof}
By degree, the part $D$ of the intersection which has multiplicity bigger than one can only be a line
or a conic. Since $D$ is defined over $\RR$, and any line contains real points, 
and the exceptional quadric does not have real points, $D$ is not a line. If $D$ were reducible,
then it is the union of two lines. If the two lines intersect, then the intersection point
would be real, but the exceptional quadric $N=0$ contains no real points; a contradiction. If the two lines
do not intersect, then we have a complete intersection of dimension~1 which is disconnected,
and this is also in contradiction to a well-known theorem in algebraic geometry
(see \cite{Hartshorne:62}). Therefore we remain with the case stated in Corollary \ref{cor:1}. 
\hfill $\BewEnde$
\end{proof}


\section{Computations in Study and Euler parameter space}\label{comp}

Within this section we prove computationally that reconstruction 2 and 3 can 
only have a 2-dimensional self-motion in the case already known to Duporcq. 

\subsection{Computation of $K_e$}

We parametrize the base as follows:
\begin{equation}
\begin{split}
\Vkt M_1&:=(0,0,0)^T,\quad \Vkt M_2:=(1,0,0)^T, \\
\Vkt M_4&:=(A_4,B_4,0)^T,\quad \Vkt M_5:=(A_5,B_5,0)^T.
\end{split}
\end{equation}
Then $\go M_3$ is already determined as the intersection point of 
$\go M_1\go M_2$ and $\go M_4\go M_5$, thus we get:
\begin{equation}
\Vkt M_3:=\left(\frac{B_4A_5-A_4B_5}{B_4-B_5},0,0\right)^T.
\end{equation}
As $\go M_3$ has to be a finite point which is not allowed to collapse with one of the other four given points we have
$B_4B_5U_1\neq 0$ with:
\begin{equation}
U_1:=(B_4-B_5)(B_4A_5-A_4B_5)(B_4A_5-A_4B_5-B_4+B_5).
\end{equation}
Now we compute the platform with the help of the mapping $\kappa_i$. In the remainder of this section 
we restrict to the case $i=2$ as the case $i=3$ can be done in a total analogous way. 
W.l.o.g.\ we can assume that the matrix $\Vkt A$ of $\kappa_2$ has the form:
\begin{equation}\label{matA}
\Vkt A:=\begin{pmatrix}
\mu_1 & \mu_2 \\
0 & \mu_3
\end{pmatrix} \quad \text{with} \quad \mu_1\mu_3\neq 0\quad \text{and}\quad \mu_1>0.
\end{equation} 
Therefore we get $\Vkt m_j=\Vkt A\Vkt M_j$ for $j=1,2,4,5$ and 
obtain $\go m_3$ as the intersection point of $\go m_2\go m_4$ and $\go m_1\go m_5$, which yields: 
\begin{equation}
\Vkt m_3:=\left(
\frac{B_4(A_5\mu_1+B_5\mu_2)}{B_4A_5+B_5-A_4B_5}, 
\frac{B_4B_5\mu_3}{B_4A_5+B_5-A_4B_5},0\right)^T.
\end{equation}
As this point also has to be a finite point we get additionally the assumption $U_2\neq 0$ with
\begin{equation}
U_2:=B_4A_5+B_5-A_4B_5.
\end{equation}
Moreover $\go M_3$ cannot be located on $\go M_4\go M_5$ which yields $U_3\neq 0$ with
\begin{equation}
U_3:=B_4A_5-B_4-A_4B_5.
\end{equation}
Let us denote the numerator of the difference $Q_1-Q_i$ ($i=2,\ldots ,5)$ of sphere conditions by 
$\Delta_i$. Then we can compute the linear combination 
\begin{equation}
B_4B_5U_1\Delta_2+U_3\Delta_3+B_5U_1U_2\Delta_4-B_4U_1U_2\Delta_5
\end{equation}
which yields a quadratic expression $K_e[1356]$ in the Euler parameters 
(free of Study parameters $f_0,\ldots ,f_3$), where the number in the bracket gives the number of terms. 

\begin{rmk}
$K_e$ cannot be fulfilled identically for the following reason: In this case the sphere conditions $Q_1, \ldots, Q_5$ 
are linearly dependent and therefore we would end up with an degenerated architectural singular manipulator  \cite{karger_1998}, 
which has to have 4 collinear points in the platform or the base (see also \cite{karger_nonplanar})
contradicting the design under consideration. \hfill $\diamond$
\end{rmk} 

\subsection{Determining the pentapod's geometry}

In the following we show that the projection $K\to K_e$ cannot be birational. This is done by contradiction; i.e.\ we assume 
that $K\to K_e$ is birational, and show that $K_e$ cannot intersect $N=0$  totally tangential along an 
irreducible quadric (cf.\ Corollary \ref{cor:1}). 

If $K_e$ touches $N=0$ along a quadric then there has to exist a double-counted plane $\varepsilon$: 
$\nu_0e_0+\nu_1e_1+\nu_2e_2+\nu_3e_3=0$ within the pencil of quadrics 
spanned by $K_e=0$ and $N=0$. Thus we can make the following ansatz $W=0$ with:
\begin{equation}
W:=K_e+\nu N + (\nu_0e_0+\nu_1e_1+\nu_2e_2+\nu_3e_3)^2.
\end{equation}
In the following we denote the coefficient of $e_0^{i}e_1^{j}e_2^{k}e_3^{l}$ of $W$ by 
$W_{ijkl}$. We consider:
\begin{equation}\label{step1}
\begin{split}
W_{1100}&=2\nu_0\nu_1,\quad W_{1010}=2\nu_0\nu_2, \\
W_{0101}&=2\nu_1\nu_3,\quad W_{0011}=2\nu_2\nu_3. 
\end{split}
\end{equation}
which implies the following two cases:
\begin{enumerate}[$\bullet$]
\item
$\nu_0=\nu_3=0$: We compute 
\begin{equation}
W_{0200}-W_{2000}=\nu_1^2, \quad 
W_{0020}-W_{0002}=\nu_2^2, 
\end{equation}
which implies that all $\nu_0,\ldots ,\nu_3$ are equal to zero, a contradiction. 
\item
$\nu_1=\nu_2=0$: Now we compute 
\begin{equation}
W_{2000}-W_{0200}=\nu_0^2, \quad 
W_{0002}-W_{0020}=\nu_3^2, 
\end{equation}
yielding the same contradiction. 
\end{enumerate}
This shows that the projection  $K\to K_e$ cannot be birational, which is equivalent with 
the condition that the system of equations $S=\Delta_2=\Delta_3=\Delta_4=\Delta_5=0$ linear 
in $f_0,\ldots ,f_3$  is linear dependent (rank of the coefficient matrix is less than 4). 
By means of linear algebra it can easily be verified that this can only be the case if 
$T=0$ holds with 
\begin{equation}
T:=\epsilon_{01}e_0e_1+\epsilon_{02}e_0e_2+\epsilon_{13}e_1e_3+\epsilon_{23}e_2e_3
\end{equation}
and
\begin{align}
\epsilon_{01}&:=\mu_3(1+\mu_1)B, &\quad
\epsilon_{02}&:=\mu_1A(\mu_3+1)-\mu_2B, \\
\epsilon_{23}&:=\mu_3(1-\mu_1)B,  &\quad
\epsilon_{13}&:=\mu_1A(\mu_3-1)+\mu_2B, 
\end{align}
by using the following abbreviations:
\begin{equation}
A:=A_5-A_4+1, \quad B:=B_4-B_5.
\end{equation}

Note that $T=0$ is also a quadric in the Euler parameter space. In the general case $T=0$ and $K_e=0$ intersect along a curve, 
but due to Corollary \ref{cor:0} they have to possess at least a common 2-dimensional component (i.e.\ a plane) or they are even 
identical.

Necessary conditions for this circumstance are obtained by determining the intersection of $K_e=T=N=0$ by resultant method;  
in detail this works as follows: We compute the resultant $R_{K_e}$ of $T$ and $N$ with respect to $e_0$. In the same way we compute 
the resultants $R_{T}$ and $R_{N}$. Then we calculate the resultant $R_{K_eT}$ of $R_{K_e}$ and $R_T$ with respect to $e_3$. 
Analogously we obtain  $R_{K_eN}$ and  $R_{TN}$. Finally we compute the greatest common divisor of $R_{K_eT}$, $R_{K_eN}$ and  $R_{TN}$, 
which can only vanish for $B_4B_5U_1U_2\mu_1\mu_3F_1F_2=0$ with 
\begin{equation}
\begin{split}
F_1:=&[B^2\mu_2-AB(\mu_1+\mu_3)](e_2^2-e_1^2) + \\
&[2A^2\mu_1-2B^2\mu_3-2AB\mu_2]e_1e_2
\end{split}
\end{equation}
and
\begin{equation}
F_2:=(1+\mu_1)(\mu_3-1)e_1^2+(1+\mu_3)(\mu_1-1)e_2^2-2\mu_2e_1e_2.
\end{equation}
In order that the three quadrics $N=T=K_e=0$ have a curve (projected bonding curve) 
in common either $F_1$ or $F_2$ has to be fulfilled identically. 
\begin{enumerate}[{{ad}}]
\item
$F_1$: We solve the coefficient of $e_2^2$ of $F_1$ for $\mu_2$ and plug the obtained expression in the coefficient of $e_1e_2$ of $F_1$. 
The resulting expression has only the real solution $A=B=0$, a contradiction. 
\item
$F_2$: It can easily be seen that $F_2$ is fulfilled identically if and only if $\mu_1=\mu_3=1$ and $\mu_2=0$ holds (due to our 
assumptions with respect to $\Vkt A$ of Eq.\ (\ref{matA})). 
\end{enumerate} 

Therefore $\kappa_2$ has to be the identity, which already implies the geometric properties of the Duporcq pentapods. 
We only remain to show that these pentapods possess 2-dimensional self-motions, 
which are line-symmetric in addition. 

\subsection{Determining the pentapod's self-motion}

Plugging $\mu_1=\mu_3=1$ and $\mu_2=0$ in $T=0$ yields:
$-2e_0(Be_1+Ae_2)=0$. Therefore we have to distinguish two cases:
\begin{enumerate}[$\bullet$]
\item
$e_1=-Ae_2/B$: Now $K_e$ possesses $1397$ terms and has to vanish identically as $dim(K_e)=2$ has to hold 
according to Corollary  \ref{cor:0}. 
We consider the coefficient of $e_0e_3$ of $K_e$ which equals $-8B_4B_5U_1U_2$ and cannot vanish without contradiction. 
\item
$e_0=0$: Now $K_e$ factors into $(e_1^2+e_2^2+e_3^2)G[185]$ where $G$ is of the form
$g_0+g_1R_1^2+g_2R_2^2+g_3R_3^2+g_4R_4^2+g_5R_5^2$, 
where all $g_i$ are functions in the geometry of the platform and the base.  
This equation can be solved for $R_3^2$ w.l.o.g.\ as $g_3$ equals $B^2U_2^2U_3$. 

Now we compute $f_0,f_1,f_3$ from $S=\Delta_2=\Delta_4=0$ w.l.o.g.. Plugging the obtained 
expressions into $\Delta_3$ and $\Delta_5$ imply in the numerators the following expressions:
\begin{equation}
-B_4U_1U_2(e_1^2+e_2^2+e_3^2)H \quad\text{and}\quad (e_1^2+e_2^2+e_3^2)H,
\end{equation}
respectively, with $H:=h_1e_1+h_2e_2$ and
\begin{equation}
\begin{split}
h_1&:=(R_2^2-R_5^2)A_4+(R_4^2-R_1^2)(A_5-1), \\
h_2&:=(R_2^2-R_5^2)B_4+(R_4^2-R_1^2)B_5.
\end{split}
\end{equation}
As $dim(K_e)=2$ has to hold according to Corollary  \ref{cor:0} the expression $H$ has to be 
fulfilled identically. 
Under the assumption $R_1^2\neq R_4^2$ we can compute $A_5$ and $B_5$ from 
$h_1=h_2=0$ but then we get $U_3=0$; a contradiction. As a consequence 
$R_1^2= R_4^2$ has to hold, which implies together with $h_1=h_2=0$ the condition $R_2^2= R_5^2$. 
This already yields a 2-dimensional self-motion. Moreover, due to $R_1^2= R_4^2$ we get $f_0=0$, 
which proves the line-symmetry of this self-motion (cf.\  \cite[Section 1]{icosapod}).
\end{enumerate}

\begin{rmk}
It should be noted that due to the existence of the similarity bond and 
$\kappa_2=id$ the Duporcq pentapods also have pure translational 1-dimensional 
self-motions (cf.\ proof of Corollary \ref{cor:0}). Moreover each 2-dimensional self-motion of a Duporcq pentapod 
contains a pure translational 1-dimensional sub-self-motion (obtained by $e_1=e_2=0$). \hfill $\diamond$
\end{rmk} 


\section{Conclusions}\label{sec:conclusion}

In light of the results of this paper, Theorem 4 of \cite{asme_ns} is not correct, but the 
flaw can be fixed by rewriting the phrase "which is not listed in Theorems 2 and 3" by 
"which is neither a Duporcq pentapod nor listed in Theorems 2 and 3". 

Furthermore we have to check if the Duporcq pentapods do not imply a further case in 
the list of non-architecturally singular hexapods with $2$-dimensional self-motions 
given in \cite[Theorem 5]{asme_ns}. Starting with a Duprocq pentapod, the two 
complete quadrilaterals are already determined and there is only one further point which 
has the same geometric properties with respect to this quadrilateral as the third anchor 
point. This is exactly the sixth anchor point illustrated in Fig.\ \ref{fig:complete_quadrilaterals}. 
But the resulting hexapod is architecturally singular as already mentioned in Subsection \ref{reason}, 
thus Theorem 5 of \cite{asme_ns} is correct. 


\section*{Acknowledgments}
The first author's research is funded by the Austrian Science Fund (FWF): P24927-N25 - ``Stewart Gough platforms with self-motions''. 
The second author's research is supported by the Austrian Science Fund (FWF): 
W1214-N15/DK9 and P26607 - ``Algebraic Methods in Kinematics: Motion Factorisation and Bond Theory''.

\end{document}